\documentclass{article}

\usepackage[preprint]{neurips_2026}

\usepackage[utf8]{inputenc}
\usepackage[T1]{fontenc}
\usepackage{hyperref}
\usepackage{url}
\usepackage{booktabs}
\usepackage{graphicx}
\usepackage{amsmath}
\usepackage{amssymb}
\usepackage{amsfonts}
\usepackage{amsthm}
\usepackage{nicefrac}
\usepackage{microtype}
\usepackage{xcolor}
\usepackage{tikz}
\usepackage{float}
\usetikzlibrary{arrows.meta,positioning,fit,calc}

\newtheorem{lemma}{Lemma}
\newtheorem{definition}{Definition}

\title{State commitment learning: training language models to distinguish computation from memory}

\author{%
  Fei Ding\textsuperscript{1}\thanks{Corresponding author: \texttt{dignfei@gmail.com}} \quad
  Yongkang Zhang\textsuperscript{1} \quad
  Runhao Liu\textsuperscript{1} \\
  Yuhao Liao\textsuperscript{2} \quad
  Zijian Zeng\textsuperscript{2} \quad
  Huiming Yang\textsuperscript{2} \\
  \textsuperscript{1}Alibaba Group \quad
  \textsuperscript{2}Tsinghua University
}

\begin{document}

\maketitle

\begin{abstract}
Reasoning language models do not distinguish tokens used for \emph{computation} from tokens that constitute \emph{persistent state}: once generated, all hidden thoughts remain in context and influence future predictions. As a result, downstream reasoning may depend on failed attempts, dead ends, and private scratch work that should not be safely relied on later. We recast this phenomenon as a new training objective, \textbf{state commitment learning}: training models to explicitly distinguish information that should be committed as persistent state from temporary computation that can be discarded. We define a counterfactual criterion, \emph{persistent-state sufficiency}, which makes it trainable and measurable whether an answer remains usable after hidden thoughts are erased. We then propose \textbf{Counterfactual Erasure RL} (CERL), which evaluates, under the same prefix, both a path that keeps hidden thoughts and a path that erases them, and gives reward only when the erasure path remains correct. We also introduce the \emph{Erasure Dependence Protocol} and show across mathematics, long-chain logic, scientific QA, and multi-turn tool-use evaluation that CERL substantially reduces answer dependence on hidden thoughts without sacrificing accuracy, consistently outperforming correctness-only RL and long-answer SFT baselines.
\end{abstract}

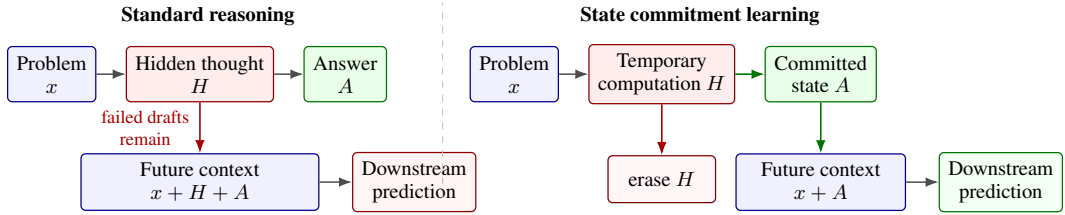
\begin{figure}[H]
\centering
\resizebox{\textwidth}{!}{%
\begin{tikzpicture}[
  font=\small,
  box/.style={draw, rounded corners=2pt, line width=0.55pt, align=center,
    inner sep=4pt, minimum height=0.72cm},
  ctx/.style={box, fill=blue!6, draw=blue!55!black},
  thought/.style={box, fill=red!7, draw=red!55!black},
  state/.style={box, fill=green!9, draw=green!45!black},
  arr/.style={-{Latex[length=2mm]}, line width=0.65pt, black!70},
  bad/.style={-{Latex[length=2mm]}, line width=0.7pt, red!65!black},
  good/.style={-{Latex[length=2mm]}, line width=0.7pt, green!45!black},
  note/.style={font=\footnotesize, align=center},
  title/.style={font=\small\bfseries}
]
\node[title] at (-3.8,1.35) {Standard reasoning};
\node[ctx] (sx) at (-6.2,0.45) {Problem\\$x$};
\node[thought, right=0.45cm of sx] (sh) {Hidden thought\\$H$};
\node[state, right=0.45cm of sh] (sa) {Answer\\$A$};
\node[ctx, below=0.78cm of sh, minimum width=3.65cm] (sctx) {Future context\\$x + H + A$};
\node[box, fill=red!4, draw=red!55!black, right=0.5cm of sctx] (sdown) {Downstream\\prediction};
\draw[arr] (sx) -- (sh);
\draw[arr] (sh) -- (sa);
\draw[bad] (sh) -- node[note, left=1pt] {failed drafts\\remain} (sctx);
\draw[arr] (sctx) -- (sdown);

\draw[black!25, dashed] (-0.2,-1.45) -- (-0.2,1.55);

\node[title] at (3.75,1.35) {State commitment learning};
\node[ctx] (rx) at (0.9,0.45) {Problem\\$x$};
\node[thought, right=0.45cm of rx] (rh) {Temporary\\computation $H$};
\node[state, right=0.45cm of rh] (ra) {Committed\\state $A$};
\node[ctx, below=0.78cm of ra, minimum width=2.6cm] (rctx) {Future context\\$x + A$};
\node[box, fill=green!5, draw=green!45!black, right=0.5cm of rctx] (rdown) {Downstream\\prediction};
\node[box, fill=red!3, draw=red!50!black, below=0.78cm of rh, minimum width=1.65cm] (rerase) {erase $H$};
\draw[arr] (rx) -- (rh);
\draw[good] (rh) -- (ra);
\draw[bad] (rh) -- (rerase);
\draw[good] (ra) -- (rctx);
\draw[arr] (rctx) -- (rdown);
\end{tikzpicture}}
\caption{Overview of state commitment learning. Standard reasoning leaves hidden thoughts in context, whereas CERL trains the model to use hidden thoughts as temporary computation, commit only the future-relevant answer state, and remain correct after erasure.}
\label{fig:overview}
\end{figure}

\section{Introduction}

\subsection{Autoregressive language models do not distinguish computation from memory}

Autoregressive language models do not distinguish tokens used for \emph{computation} from tokens used for \emph{memory}: once a reasoning token is generated, it is appended to the context and becomes input to future predictions. This default makes long reasoning traces both a source of computational burden and a source of spurious dependence. Downstream predictions may depend on failed attempts, dead ends, or private scratch work that should not enter the model's persistent state.

Recent evidence suggests that longer chains of thought are not always beneficial: some reasoning tokens provide no positive contribution to final or downstream correctness and may even interfere with future predictions (see \S\ref{sec:rw-long-cot}). This suggests that retaining every generated reasoning token as context is not a reasonable memory mechanism.

\subsection{The state commitment problem and the internalization gap}

We recast the above phenomenon as a training-objective problem that has been less explicitly modeled. Reasoning models lack a \emph{state commitment} mechanism: after using hidden computation, a model should explicitly commit the information needed for future reasoning into a persistent answer state while discarding the remaining temporary computation.

Existing mitigations mostly operate at inference time on already generated thought traces, such as eviction, pruning, compression, or selecting shorter traces. They do not directly train the model itself to decide what is persistent and what is temporary. Their common blind spot is that they never optimize the property that ``answer $A$ remains a usable state interface after hidden thought $H$ is erased.'' We call this gap the \emph{internalization gap}.

\subsection{Method overview}

The core distinction from prior work can be summarized as follows: \emph{compression asks how to preserve the past trace; state commitment asks what the future needs}. Compression treats the reasoning trace as a past object to be compactly preserved. State commitment treats the visible answer as a future-facing state interface that must be explicitly committed.

We instantiate this idea with \textbf{Counterfactual Erasure RL} (CERL). Under the same problem prefix, CERL evaluates matched continuations in parallel: (i) a full-thought path that keeps hidden thoughts, (ii) an erasure path that progressively deletes generated hidden thoughts, and (iii) a skip baseline used to audit recomputation. The main reward is tied to erasure-path correctness, with auxiliary length and anti-postponement controls. This signal directly captures whether answer $A$ can support downstream reasoning when hidden thought $H$ is no longer visible.

Because prior work has not made this property a training objective, naively dropping thoughts at inference time often collapses accuracy. Training under counterfactual erasure instead lets the answer state itself serve as the persistent interface while preserving accuracy.

\subsection{Contributions}

Our contributions follow five levels: a new problem, a new objective, a new method, a new protocol, and new empirical findings.

\noindent\textbf{(1) New problem definition.} We propose \emph{State Commitment Learning}, which captures the missing training objective that reasoning models should distinguish temporary computation from persistent state that can be relied on in the future. We give a concrete setting, \emph{Post-Answer Hidden-Thought Erasure}: the model alternates between hidden thoughts and visible answer states; each thought is progressively erased when the next thought begins; the final context retains only visible answers. We make no a priori assumption about the content relation between $A$ and $H$; it emerges from the training objective.

\noindent\textbf{(2) New counterfactual objective: persistent-state sufficiency.} We define the counterfactual criterion of \emph{persistent-state sufficiency}: under matched full-thought / erasure paths, erasing hidden thought should not reduce downstream correctness. This turns whether an answer has been committed into a future-reliable state into an optimizable and estimable condition; the paired counterfactual paths are formalized in Definition~\ref{def:pss}.

\noindent\textbf{(3) New training algorithm: CERL+HSCO.} We propose Counterfactual Erasure RL as an optimization algorithm for persistent-state sufficiency, implemented as \emph{Hierarchical State-Commitment Optimization} (HSCO): two GRPO layers train the hidden-thought policy $H_1$ and the answer state $A_1$ over non-overlapping token ranges.

\noindent\textbf{(4) New evaluation protocol: erasure dependence.} We propose the \emph{Erasure Dependence Protocol}. It reports four sufficiency metrics after erasure: \emph{Answer Sufficiency Gap} (ASG, the accuracy gap between full and erasure paths, lower is better), \emph{Hidden Thought Dependency Rate} (HTDR, the fraction of examples where the full path is correct but the erasure path is wrong, lower is better), \emph{Erasure Success Rate} (ESR, conditional correctness after erasure given full-path correctness, higher is better), and \emph{Answer Interface Sufficiency} (AIS, the erasure-path accuracy divided by full-path accuracy, higher is better). We also report \emph{Marginal State Gain} (MSG),
\[
\mathrm{MSG} \;=\; \mathrm{Acc}(x + A_1) \;-\; \mathrm{Acc}(x + \texttt{empty}) \;=\; \mathrm{Acc}_{\mathrm{pe}} - \mathrm{Acc}_{\mathrm{skip}} .
\]
MSG rules out the narrow explanation that the erasure path succeeds by ignoring $A_1$ and recomputing from $x$; it should still be interpreted together with downstream length control and visible-CoT leakage audits.

\noindent\textbf{(5) Empirical findings.} Our experiments are designed to show that correctness-only reinforcement learning is insufficient for learning persistent-state sufficiency, while counterfactual erasure training substantially closes ASG without sacrificing accuracy and improves state management in BFCL-v3 multi-turn tool use. Long-answer SFT is used as a diagnostic control for the weaker alternative explanation that supervising longer visible answers alone naturally yields the same erasure-dependence signature; it is not used as a standalone causal attribution.

\section{Related work}

\subsection{Long-CoT degradation as motivation}
\label{sec:rw-long-cot}

A growing body of work reports that longer chains of thought can be worse~\citep{hassid2026dontoverthinkitpreferring,wu2026when,luo-etal-2025-valley,zheng2025the}. These observations motivate our central question: if many reasoning tokens do not help final or downstream reasoning, should every generated thought token be retained as future context?

\subsection{Inference-time pruning and compression do not internalize the boundary}
\label{sec:rw-inference-time}
\label{sec:rw-vs-compression}
\label{sec:rw-cognitive}

Existing methods mostly manage already generated traces through pruning, eviction, compression, selection, or external memory. These approaches ask how to control or preserve past reasoning traces. Our question is different: whether the model can be trained so that the visible answer state $A$ remains a sufficient downstream interface after hidden thought $H$ is erased. Detailed positioning is in Appendix~\ref{app:direction}.

\subsection{Long-answer supervision is not state commitment}
\label{sec:rw-vs-cot-in-answer}

Long-answer SFT is a useful diagnostic control, but it supervises answer form rather than counterfactual sufficiency. It tests whether longer visible answer states alone can reproduce the erasure-dependence signature. In contrast, CERL optimizes whether $A$ remains reliable after $H$ is erased; $A$ may be short or long, and form is not the objective.

\section{Method}

\subsection{Problem definition: post-answer hidden-thought erasure}
\label{sec:problem}

Given a problem $x$, the model may generate hidden thought $H$ to help produce answer $A$. Before the next hidden-thought segment begins, the previous $H$ is erased, and only the committed $x+A$ remains for downstream context. The central question is whether the model can learn, during training, to make $A$ correct, sufficient, and future-reliable after $H$ is erased; equivalently, whether it can explicitly distinguish the boundary between computation and state.

\paragraph{Formal definition of progressive erasure.}
\begin{verbatim}
For t = 1..n:
    H_t = Think(x, A_{1:t-1})
    A_t = Output_State(x, A_{1:t-1}, H_t)
    Before H_{t+1} starts generating: Erase(H_t)
\end{verbatim}
The final context is $x + A_1 + \cdots + A_n = x + A_{1:n}$, and all think segments have disappeared.

$A_1,\ldots,A_n$ are all permanently retained answer tokens, and each $A_i$ is produced with the help of its corresponding hidden thought $H_i$. We do not model whether to enter thought as a separate learning target: if an intermediate step requires no hidden computation, it does not form a separate erasure segment and is treated as part of a neighboring answer state. Form-level differences are not prescribed; they are determined by the objective.

\paragraph{Sufficiency target.}
We instantiate persistent-state sufficiency with two requirements:
\begin{itemize}
    \item In single-turn settings, $x + A_{1:n}$ should be sufficient to infer the final correct answer without $H$.
    \item In multi-turn settings, the historical visible answers $A_{1:n}$ should support follow-up queries after all $H$ has been erased.
\end{itemize}

\subsection{Tag semantics}

Visible text is the default output channel. Only \verb|<think>...</think>| is special and marks temporary hidden-thought spans. The generation alternates between hidden-thought spans and answer states; a full example appears in Appendix~\ref{app:erasure-example}.

\subsection{Training objective: persistent-state sufficiency}
\label{sec:objective}

We directly define a future-facing objective in task-projected terms. Given the same generated tuple $(x,H,A)$, if $A$ has already served as the persistent-state interface, then erasing $H$ should not reduce downstream correctness:
\[
\Pr_\pi[C_{\mathrm{pe}}=1 \mid x,H,A]
\;\ge\;
\Pr_\pi[C_{\mathrm{full}}=1 \mid x,H,A].
\]
Here $C_{\mathrm{full}}$ and $C_{\mathrm{pe}}$ denote downstream correctness events under the full-thought and erasure paths, respectively. The meaning is that $A$, as committed persistent state, should serve as the interface supporting downstream reasoning after hidden thought $H$ is erased. In this sense, $A$ crosses the computation--state boundary and becomes an object that can be relied on in the future. CERL's binary correctness reward is a finite-sample surrogate for this objective.

We make no content-level assumption about $A$ and $H$: $A$ may inherit, rewrite, extend, or be independent of $H$. Its final form emerges from the training objective rather than from a form constraint. The commitment is an objective-level property, persistent-state sufficiency, not a form-level constraint.

\subsection{Theoretical framework}
\label{sec:theory}

This section defines the estimable objects used for training and evaluation. We do not define PSS as equality of full output distributions; instead, we define its task-projected version over evaluator-induced binary correctness events. To avoid notation overload, $P_\pi(y\mid\cdot)$ denotes only the output distribution of policy $\pi$; $\Pr_\pi[\cdot]$ denotes event probabilities induced jointly by the policy, data distribution, decoder randomness, and evaluator; and $\mathrm{Acc}$ abbreviates the probability of a correctness event. The Acc values in experimental tables are finite-sample estimates. Specifically,
\[
\mathrm{Acc}_{\text{full}}(\pi) = \Pr_\pi[C_{\text{full}}=1],
\qquad
\mathrm{Acc}_{\text{pe}}(\pi) = \Pr_\pi[C_{\text{pe}}=1],
\]
where $C_{\text{full}}$ and $C_{\text{pe}}$ denote final-answer correctness under the full-thought and progressive-erasure paths.
The full / erasure comparison is paired counterfactual evaluation: we first fix the same generated tuple $(x,H,A)$ and then construct two evaluation contexts
\[
g_{\mathrm{full}}(x,H,A)=(x,H,A),\qquad
g_{\mathrm{pe}}(x,H,A)=(x,A).
\]
Thus $A$ is identical in the two paths, and the only intervention is whether $H$ remains in the future context. We compare the effect of this intervention on correctness events; we do not require the two paths to generate identical downstream text or induce identical output distributions.

\begin{definition}[Persistent-state sufficiency]
\label{def:pss}
A policy $\pi$ satisfies \emph{persistent-state sufficiency} (PSS) iff, under the paired context intervention above, erasing $H$ does not reduce downstream correctness:
\[
\Pr_\pi[C^{\mathrm{pe}}=1 \mid x,H,A]
\;\ge\;
\Pr_\pi[C^{\mathrm{full}}=1 \mid x,H,A] .
\]
\end{definition}

\begin{definition}[Hidden-thought dependence]
\label{def:dh}
\[
D_H(\pi) \;=\; \Pr_\pi\bigl[C_{\text{full}}=1 \;\wedge\; C_{\text{pe}}=0\bigr].
\]
$D_H$ is a sample-level violation rate of PSS: $D_H=0$ means that no sample is correct on the full path and wrong on the erasure path. Larger $D_H$ means downstream reasoning relies more on the erased computation side ($H$) than on the memory side ($A$). Its empirical estimate is HTDR (\S\ref{sec:metrics}).
\end{definition}

\begin{lemma}[Relationship between ASG and $D_H$]
\label{lem:asg}
Let the erasure purification rate be $D_E(\pi)=\Pr_\pi[C_{\text{full}}=0 \wedge C_{\text{pe}}=1]$. Then
\[
\mathrm{ASG} \;=\; \mathrm{Acc}_{\text{full}} - \mathrm{Acc}_{\text{pe}} \;=\; D_H - D_E,
\]
so $\mathrm{ASG}\le D_H$, with equality when $D_E=0$. $\mathrm{ASG}<0$ means $D_E>D_H$, i.e., the purification effect of erasure exceeds the cost of state-commitment failure.
\end{lemma}

The proof is in Appendix~\ref{app:theory-proof}. This decomposition shows that ASG is not a standard accuracy metric; it is the difference between hidden-thought dependence $D_H$ and erasure purification $D_E$. Reporting full-path accuracy alone cannot determine whether the answer state is sufficient after erasure. We therefore explicitly include $\mathrm{Acc}_{\text{pe}}$, HTDR, ESR, and MSG in training and evaluation. In implementation, $c_{\text{full}}$ and $c_{\text{pe}}$ are Monte Carlo estimates from repeated rollouts under the two paired contexts above.

\subsection{Training signal: CERL and HSCO}
\label{sec:cerl}

We implement erasure-dependence training at two complementary levels. CERL optimizes a finite-sample surrogate of persistent-state sufficiency through paired full-thought / erasure dual-context rollouts. HSCO uses two GRPO layers to update hidden thought $H_1$ and answer state $A_1$. Concretely, $c_{\text{pe}}$ is a Monte Carlo estimate of $\Pr_\pi[C_{\mathrm{pe}}=1\mid x,H,A]$, so the binary correctness reward directly corresponds to the task-projected PSS objective above, rather than to full output-distribution matching. For each $x$, we sample 4 candidates for $H_1$; for each $H_1$, we sample 8 candidates for $A_1$. Each $A_1$ is evaluated under both a full-thought path that retains $H_1$ and an erasure path that deletes $H_1$, plus an additional skip baseline. Full sampling pseudocode and statistics are in Appendix~\ref{app:cerl-details}.

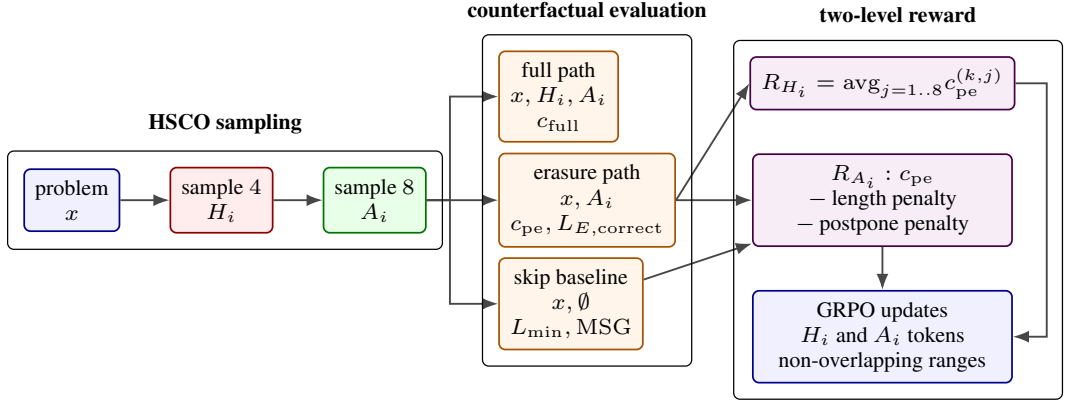
\begin{figure}[t]
\centering
\resizebox{\textwidth}{!}{%
\begin{tikzpicture}[
  font=\scriptsize,
  box/.style={draw, rounded corners=2pt, line width=0.55pt, align=center,
    inner sep=3.5pt, minimum height=0.58cm},
  data/.style={box, fill=blue!6, draw=blue!55!black},
  thought/.style={box, fill=red!7, draw=red!55!black},
  state/.style={box, fill=green!9, draw=green!45!black},
  eval/.style={box, fill=orange!8, draw=orange!65!black},
  reward/.style={box, fill=violet!7, draw=violet!55!black},
  arr/.style={-{Latex[length=1.8mm]}, line width=0.6pt, black!72},
  note/.style={font=\scriptsize, align=center}
]
\node[data] (x) at (0,0) {problem\\$x$};
\node[thought, right=0.55cm of x] (h) {sample $4$\\$H_i$};
\node[state, right=0.55cm of h] (a) {sample $8$\\$A_i$};

\node[eval, right=0.8cm of a, yshift=1.16cm] (full) {full path\\$x,H_i,A_i$\\$c_{\rm full}$};
\node[eval, right=0.8cm of a] (erase) {erasure path\\$x,A_i$\\$c_{\rm pe},L_{E,\rm correct}$};
\node[eval, right=0.8cm of a, yshift=-1.16cm] (skip) {skip baseline\\$x,\emptyset$\\$L_{\min},\mathrm{MSG}$};

\node[reward, right=0.85cm of erase, minimum width=2.9cm] (rw) {$R_{A_i}: c_{\rm pe}$\\$-\,$length penalty\\$-\,$postpone penalty};
\node[reward, above=0.48cm of rw, minimum width=2.9cm] (hrw) {$R_{H_i}=\mathrm{avg}_{j=1..8}c_{\rm pe}^{(k,j)}$};
\node[data, below=0.48cm of rw, minimum width=2.9cm] (upd) {GRPO updates\\$H_i$ and $A_i$ tokens\\non-overlapping ranges};
\coordinate (hrwroute) at ([xshift=0.35cm]hrw.east);

\draw[arr] (x) -- (h);
\draw[arr] (h) -- (a);
\draw[arr] (a.east) -- ++(0.26,0) |- (full.west);
\draw[arr] (a.east) -- (erase.west);
\draw[arr] (a.east) -- ++(0.26,0) |- (skip.west);
\draw[arr] (erase.east) -- (hrw.west);
\draw[arr] (erase) -- (rw);
\draw[arr] (skip) -- (rw);
\draw[arr] (rw) -- (upd);
\draw[arr] (hrw.east) -- (hrwroute) |- (upd.east);
\node[draw, rounded corners=2pt, fit=(x)(h)(a), inner sep=5pt, label={[font=\scriptsize\bfseries, yshift=2pt]above:HSCO sampling}] {};
\node[draw, rounded corners=2pt, fit=(full)(erase)(skip), inner sep=5pt, label={[font=\scriptsize\bfseries, yshift=2pt]above:counterfactual evaluation}] {};
\node[draw, rounded corners=2pt, fit=(hrw)(rw)(upd)(hrwroute), inner sep=5pt, label={[font=\scriptsize\bfseries, yshift=2pt]above:two-level reward}] {};
\end{tikzpicture}}
\caption{CERL/HSCO training flow. Each candidate answer state is evaluated under matched full, erasure, and skip contexts; the erasure path supplies the main sufficiency reward, while the skip baseline and length terms discourage recomputation and postponed computation.}
\label{fig:cerl-hsco}
\end{figure}

For each $A_1$, let $c_{\text{full}}$ and $c_{\text{pe}}$ denote correctness rates under the full-thought and erasure paths, and let $L_E$ be the average downstream length after $A_1$ under the erasure path. Let $L_{F,\text{correct}}$, $L_{E,\text{correct}}$, and $L_{S,\text{correct}}$ denote the downstream lengths of correct trajectories under the full-thought, erasure, and skip paths, respectively; if a path has no correct trajectory, its length is set to $+\infty$. Let $\overline{L}_{E,\text{correct}}^{(k)}$ denote the mean $L_{E,\text{correct}}$ over multiple $A_1^{(k,i)}$ candidates under the same $H_1^{(k)}$. We then define the downstream length baseline:
\begin{equation}
\label{eq:lmin}
L_{\min}^{(k,i)} = \min\bigl( L_{F,\text{correct}},\; \overline{L}_{E,\text{correct}}^{(k)},\; L_{S,\text{correct}} \bigr) .
\end{equation}
$L_{\min}$ is the successful downstream-length baseline for the current sample and is used to constrain $L_{E,\text{correct}}$; paths with no correct trajectory take length $+\infty$.

\subsubsection{\texorpdfstring{$A_1$}{A\_1}-level signal}
\label{sec:cerl-a1}

The $A_1$-level signal uses erasure-path correctness $c_{\text{pe}}$ as the main reward and adds two form-control penalties. For each $A_1^{(k,i)}$ (superscripts omitted below), let $m=\mathrm{mean}_i\,\mathrm{len}(A_1^{(k,i)})$ be the average length of the 8 $A_1$ candidates under the same $H_1$. The reward is
\begin{equation}
\label{eq:a1-reward}
R \;=\;
\begin{cases}
c_{\text{pe}} \;-\; \alpha \cdot \dfrac{\mathrm{len}(A_1)}{m} \;-\; \beta \cdot \max\!\Bigl(0,\, \dfrac{L_{E,\text{correct}}}{L_{\min}} - 1\Bigr), & \text{if } c_{\text{pe}} > 0,\\[2mm]
-\psi, & \text{otherwise}.
\end{cases}
\end{equation}
This reward learns the state-commitment boundary through erasure-after-correctness gating, rather than through a simple linear tradeoff between correctness and length. First, the main reward comes from erasure-path correctness: after $H$ is deleted, downstream reasoning can rely only on $x+A$, so if $A$ has not committed information needed later, $c_{\text{pe}}$ drops. This signal pushes future-relevant information into the visible answer state. Second, the reward is gated: when the erasure path fails, the sample receives $-\psi$, so a short but insufficient $A_1$ cannot benefit. Only when $c_{\text{pe}}>0$ do the $A_1$ length pressure and downstream anti-postpone pressure act among candidates that are already correct after erasure. The $\alpha$ term provides peer-relative length pressure within the same $H_1$, while the $\beta$ term discourages postponing critical computation until after erasure or making $A_1$ worse than the skip baseline. In other words, CERL first requires $A_1$ to be sufficient for correct post-erasure reasoning, and only then optimizes the economy of the persistent state: information that remains needed downstream is encouraged to be committed to $A_1$, while trial-and-error, exploration, and redundant computation are encouraged to remain in temporary \texttt{<think>} spans that can be erased. The hard penalty $-\psi$ also closes a reward-hacking channel in which the model could make the full-thought path fail as well to avoid penalty. Design motivation, default hyperparameters, and GRPO implementation details are in Appendix~\ref{app:reward-design}.

\subsubsection{\texorpdfstring{$H_1$}{H\_1}-level signal}
\label{sec:cerl-h1}

The 4 sampled $H_1$ candidates under the same $x$ form the $H_1$-level GRPO group, with reward
\[
R_{H_1}^{(k)} \;=\; \mathrm{avg}_{i=1..8}\, c_{\text{pe}}^{(k,i)} .
\]
Thus, under the same problem, the preferred $H_1$ is the one whose 8 $A_1$ descendants are more often correct after erasure. This signal updates hidden-thought policies that better support commit-ready answer states; gradients are applied only to $H_1$ tokens.

\subsubsection{Coordination and curriculum activation}
\label{sec:cerl-curric}

The two rewards backpropagate through non-overlapping token ranges to avoid interference. Training first activates only the $A_1$-level signal to stabilize the answer-state interface. Later end-to-end curriculum stages activate the $H_1$-level signal so that hidden-thought policy and answer state are optimized jointly. Layer structure and curriculum details are in Appendix~\ref{app:reward-design}.

\subsection{Training pipeline}
\label{sec:pipeline}

\paragraph{Stage 1: Skill data production.}
We first construct successful trajectories with a progressive-erasure timeline: each substep generates a temporary \verb|<think>|$H_i$\verb|</think>| and then commits a visible state $A_i$; before the next think segment begins, the previous $H_i$ is deleted. Each example stores both the full version and the erased version.Further details are in Appendix~\ref{app:pipeline-details}.

\paragraph{Stage 2: Answer-state commitment SFT.}
Each trajectory is split into multiple training examples according to the erasure timeline, so that each training context contains at most the current think segment and all previous think segments have already been erased. SFT therefore teaches the alternating format, answer-state commitment under hidden-thought assistance, and continued reasoning when earlier thoughts are not visible.

\paragraph{Stage 3: CERL.}
Starting from the SFT model, we sample full-thought and progressive-erasure paths in parallel and optimize erasure-path correctness with Eq.~\eqref{eq:a1-reward}. The curriculum gradually expands the erasure range and gradient range: it first stabilizes local $A_1$-level PSS and then enables full-chain erasure with both $A_1$ and $H_1$ signals.

\section{Experiments}

\subsection{Erasure dependence protocol}
\label{sec:metrics}

The \emph{Erasure Dependence Protocol} reports both task accuracy and post-erasure sufficiency. The basic accuracy metric is Acc, reported in the tables as Avg@8. The core metrics are
\[
\mathrm{ASG}=\mathrm{Acc}_{\text{full}}-\mathrm{Acc}_{\text{pe}},\quad
\mathrm{HTDR}=\Pr_\pi[C_{\text{full}}=1 \wedge C_{\text{pe}}=0],
\]
\[
\mathrm{ESR}=\Pr_\pi[C_{\text{pe}}=1 \mid C_{\text{full}}=1],\quad
\mathrm{AIS}=\mathrm{Acc}_{\text{pe}}/\mathrm{Acc}_{\text{full}} .
\]
Lower ASG is better, higher ESR/AIS is better, and HTDR is the empirical estimate of Definition~\ref{def:dh}. To rule out the alternative explanation of recomputation after erasure, we also report marginal state gain:
\[
\textbf{MSG} \text{ (Marginal State Gain)} \;=\; \mathrm{Acc}(x + A_1) \,-\, \mathrm{Acc}(x + \texttt{empty}) .
\]
$\mathrm{MSG}>0$ indicates that $A_1$ contributes nontrivial marginal state information beyond the skip-context recomputation baseline. Algebraic relations among metrics and form reports (HTT/RHTT, $\overline{\mathrm{len}}(A)$, TTF) are in Appendix~\ref{app:metrics-details}.

\subsection{Experimental setup}
\label{sec:exp-setup}

The goal is not to compare who compresses reasoning traces more aggressively, but to test the core claim: hidden thought can serve as temporary computation, and once its conclusion is committed to visible answer state, the original thought trace can be erased without harming downstream solving. Experiments therefore ask whether (i) the progressive-erasure path maintains high correctness without retaining hidden thought in the final context, and (ii) the gain comes from answer-state commitment rather than ordinary trace compression, long-answer supervision, test-time refresh, or correctness-only RL. Beyond mathematical, logical, and scientific QA tasks, we include BFCL-v3 to test whether the state interface is effective in multi-turn tool use.

\paragraph{Base model and training data.}
The main table uses Qwen3-8B~\citep{yang2025qwen3technicalreport} as the base model; Appendix~\ref{app:large-models} reports same-protocol Qwen3-14B/32B results. Training data uses strictly decontaminated DeepMath-103K~\citep{he2025deepmath103klargescalechallengingdecontaminated}, converted in Stage 1 into answer-state commitment SFT data with a progressive-erasure timeline, followed by CERL training using the objective in \S\ref{sec:cerl}. All experiments use a $C=32$k context window.

\paragraph{Evaluation sets.}
We evaluate on AIME 2024, AIME 2025, ZebraLogic~\citep{pmlr-v267-lin25i}, AutoLogi~\citep{zhu2025autologiautomatedgenerationlogic}, BFCL-v3~\citep{pmlr-v267-patil25a}, and GPQA-Diamond~\citep{rein2024gpqa}. AIME and GPQA-Diamond test difficult mathematical and scientific reasoning, ZebraLogic and AutoLogi test long-chain logic, constraint preservation, and multi-stage state updates, and BFCL-v3 (Berkeley Function-Calling Leaderboard v3) tests tool selection, parameter maintenance, and cross-turn state management in multi-turn function calling.

\paragraph{Decoding and statistics.}
All methods use Temperature=0.6, TopP=0.95, TopK=20, and MinP=0. Each problem is sampled independently 8 times and reported as Avg@8. Task accuracies are reported over 5 random seeds as mean $\pm$ 95\% bootstrap CI, while erasure-dependence metrics are aggregate point estimates over the same evaluation set. Differences from baselines use paired bootstrap tests. The main tables report task correctness and core erasure-dependence metrics; HTDR, AIS, retained hidden-thought tokens, $\overline{\mathrm{len}}(A)$, and TTF are protocol extensions defined in Appendix~\ref{app:metrics-details}.

\paragraph{Baselines.}
We group comparisons into four categories. The first contains base paths without state commitment: Base Qwen3 and Fixed-ratio erasure + free answer, which share the erasure mechanism but have no CERL training. The second contains alternative training objectives: Long-answer SFT, Correctness-only RL, and Length-penalty RL. Long-answer SFT supervises the erased version of the same successful trajectories, i.e., input $x$ and target $A_{1:n}$ only, with no $H$; it is a diagnostic control for whether supervising long visible answer states naturally yields similar behavior, not a single-variable causal comparison. The third contains trajectory-retention and test-time management baselines, including TokenSkip~\citep{xia-etal-2025-tokenskip}, trajectory compression baselines, and Halo~\citep{li2026limitedreasoningspacecage}. The fourth is CERL full. ASG/ESR/MSG are defined only for methods with a homologous multi-segment $H_i/A_i$ interface that can run both full-thought and progressive-erasure paths; ordinary baselines without this interface use ``-'' in these columns.

\subsection{Experiment 1: main results}
\label{sec:exp-main}

The main experiment runs, for each problem, both the full-thought path and the progressive-erasure path, compares final-answer correctness, and reports ASG, ESR, and MSG. The protocol requires a locatable multi-segment $H_i/A_i$ structure. Methods that produce ordinary answers, single-segment CoT, compressed traces, or external state management have no erasure points or $A_1$ skip control aligned with CERL, so these metrics are not computed. $\mathrm{Acc}_{\text{skip}}$ and MSG are jointly reported to rule out a narrow recomputation explanation: the erasure path does not succeed by completely ignoring $A_1$ and recomputing from $x$. We define $\mathrm{MSG}=\mathrm{Acc}_{\text{pe}}-\mathrm{Acc}_{\text{skip}}$; positive MSG indicates that $A_1$ has nontrivial marginal contribution to downstream reasoning, but does not by itself prove that $A_1$ is an abstract state. We therefore interpret it together with visible-CoT leakage audits, downstream length control, and form statistics.

\begin{table}[h]
\centering
\small
\caption{Main results (Qwen3-8B). Acc is Avg@8; lower ASG is better, higher ESR and MSG are better; ``-'' in ASG/ESR/MSG means the metric is not applicable.}
\label{tab:main-results}
\resizebox{\textwidth}{!}{
\begin{tabular}{llcccccc ccc}
\toprule
Method & Training & AIME24 & AIME25 & ZebraLogic & AutoLogi & GPQA-D & BFCL-v3 & ASG$\downarrow$ & ESR$\uparrow$ & MSG$\uparrow$ \\
\midrule
Base Qwen3-8B & N/A & 75.9$\pm$0.6 & 67.1$\pm$0.7 & 84.8$\pm$0.6 & 89.0$\pm$0.5 & 61.9$\pm$0.8 & 68.0$\pm$0.7 & - & - & - \\
Fixed-ratio erasure + free answer & None & 76.4$\pm$0.8 & 68.0$\pm$0.8 & 85.5$\pm$0.7 & 90.0$\pm$0.8 & 62.6$\pm$1.1 & 68.4$\pm$0.8 & - & - & - \\
Length-penalty RL & RL & 75.1$\pm$0.8 & 67.3$\pm$0.8 & 85.0$\pm$0.8 & 88.1$\pm$0.7 & 61.3$\pm$0.9 & 67.6$\pm$0.9 & - & - & - \\
TokenSkip$^\dagger$ & SFT & 72.0$\pm$0.9 & 62.9$\pm$0.8 & 78.3$\pm$1.0 & 82.1$\pm$0.9 & 55.5$\pm$0.9 & 63.7$\pm$0.6 & - & - & - \\
Halo & SFT+RL & 79.5$\pm$0.8 & 70.3$\pm$0.8 & 89.1$\pm$0.8 & 93.0$\pm$0.7 & 64.6$\pm$1.0 & 68.4$\pm$0.8 & - & - & - \\
Long-answer SFT & SFT & 78.4$\pm$0.9 & 68.2$\pm$1.0 & 85.7$\pm$0.8 & 90.8$\pm$0.7 & 62.8$\pm$1.1 & 68.3$\pm$0.7 & - & - & - \\
Correctness-only RL & RL & 79.6$\pm$0.8 & 70.8$\pm$0.9 & 89.2$\pm$0.7 & 93.1$\pm$0.6 & 64.6$\pm$0.9 & 68.5$\pm$0.6 & - & - & - \\
\textbf{CERL full} & SFT+RL & \textbf{82.1$\pm$0.6} & \textbf{74.0$\pm$0.7} & \textbf{92.4$\pm$0.5} & \textbf{94.1$\pm$0.5} & \textbf{67.4$\pm$0.8} & \textbf{76.7$\pm$0.8} & \textbf{-1.5} & \textbf{99.9} & \textbf{8.5} \\
\bottomrule
\end{tabular}}
\end{table}

The table addresses four alternative explanations. If CERL were merely ordinary CoT compression or test-time trace management, TokenSkip / compression baselines and Halo should obtain similar gains. If correctness optimization were sufficient, Correctness-only RL should approach CERL full. Long-answer SFT improves some task accuracies, but it mainly tests the weaker alternative explanation of long visible-answer supervision; because it lacks a homologous erasure interface, it does not support ASG/MSG causal attribution. CERL's unique signature is zero retained hidden-thought tokens at final evaluation, low ASG, high ESR, and high MSG, while $\overline{\mathrm{len}}(A)$ and TTF are only form reports.

\subsection{Experiment 2: training-objective ablation}
\label{sec:exp-ablation}

This experiment provides a protocol-compatible comparison of erasure dependence: all variants that report ASG/MSG share the homologous multi-segment $H_i/A_i$ interface, so differences can be more directly attributed to the training objective and curriculum. We remove RL, the progressive-erasure curriculum, and the anti-postpone penalty; these component removals retain the homologous multi-segment erasure interface and therefore report ASG/MSG.

\begin{table}[h]
\centering
\small
\caption{Training-objective ablation (Qwen3-8B). Component removals within the CERL family report Avg@8 and core erasure metrics.}
\label{tab:cerl-ablation}
\resizebox{\textwidth}{!}{
\begin{tabular}{lcccccccc}
\toprule
Variant & AIME24 & AIME25 & ZebraLogic & AutoLogi & GPQA-D & BFCL-v3 & ASG$\downarrow$ & MSG$\uparrow$ \\
\midrule
CERL full & 82.1$\pm$0.6 & 74.0$\pm$0.7 & 92.4$\pm$0.5 & 94.1$\pm$0.5 & 67.4$\pm$0.8 & 76.7$\pm$0.8 & $-1.5$ & 8.5 \\
$-$ CERL RL, SFT only & 67.0$\pm$0.7 & 58.1$\pm$0.6 & 75.7$\pm$0.8 & 80.1$\pm$0.6 & 52.9$\pm$0.7 & 59.1$\pm$0.6 & 8.8 & $-0.6$ \\
$-$ progressive erasure curriculum & 78.2$\pm$0.8 & 70.3$\pm$0.9 & 88.6$\pm$0.7 & 91.4$\pm$0.6 & 64.8$\pm$0.9 & 69.9$\pm$0.5 & 2.6 & 3.5 \\
$-$ anti-postpone penalty & 76.9$\pm$0.8 & 68.6$\pm$0.8 & 87.2$\pm$0.6 & 89.1$\pm$0.6 & 62.2$\pm$0.9 & 68.7$\pm$0.8 & 3.8 & 2.2 \\
\bottomrule
\end{tabular}}
\end{table}

Ablations show that removing CERL RL within the same training family substantially degrades ASG and MSG, indicating that SFT initialization alone does not learn answer-state commitment under erasure. Removing the progressive-erasure curriculum preserves an erasure objective but makes credit assignment less stable. Removing the anti-postpone penalty increases ASG, suggesting that the model more easily postpones critical computation until after erasure.

\subsection{Experiment 3: failure-mode audit}
\label{sec:exp-audit}

We audit the main failure modes with two controls; the full rubric and length-control details are in Appendix~\ref{app:audit-details}. The skip audit compares $\mathrm{Acc}(x+A_1)$ with $\mathrm{Acc}(x+\texttt{empty})$; $\mathrm{MSG}>0$ indicates that the erasure path does not succeed by completely ignoring $A_1$ and recomputing from $x$. The visible-CoT leakage audit checks whether $A$ contains step-by-step trial and error, self-correction, or scratch-like derivations. Downstream length control in the CERL reward further discourages compensating for uncommitted state by extending downstream reasoning after erasure.

\section{Empirical findings}
\label{sec:findings}

The experiments support three conclusions. First, hidden thought can serve as temporary computation: CERL full maintains high Acc while final retained hidden-thought tokens are zero and ASG is low, showing that once conclusions are committed, the original thought trace can be erased. Second, correctness training is insufficient for this property: within the same CERL family, removing RL or counterfactual erasure components increases ASG and decreases MSG. Third, this ability appears not only in mathematical and long-chain logical tasks, but also in multi-turn tool use; this suggests that state commitment learning is learning a more reliable downstream state interface, not merely a longer answer form. Long-answer SFT only tests the weaker explanation of ``longer visible answers,'' whereas CERL's signature is simultaneously high Acc, low ASG, and high MSG.

\section{Discussion}

\paragraph{Must long chains of thought be retained?}
The current reasoning paradigm assumes that thought traces must remain permanently in context, which conflates computation with memory. We do not question whether a model should think deeply; we question whether thought conclusions must be explicitly committed to persistent state and whether the original trace must remain available. The model can still think extensively, but thought is temporary computation. Once its conclusion is committed into the answer state, the temporary thought is no longer the object that downstream reasoning must retain.

\paragraph{Connection to agent settings.}
Long-horizon agents already require some form of summarization because context windows cannot hold full trajectories. Our method internalizes this capability into the model: rather than relying on an external summarizer, the model performs the loop ``computation $\to$ state commitment $\to$ erasure'' during generation. BFCL-v3 further indicates that this mechanism can transfer to state maintenance in multi-turn tool use, rather than being only a formatting trick for single-turn reasoning problems. This provides a training-time approach to context management for agents, rather than an inference-time post-processing step.

\section{Broader impact}

If state commitment learning scales, it may shift large-model reasoning from retaining full chains of thought to using temporary computation plus committed state. This could reduce long-context cost, improve agent robustness, reduce unnecessary retention of private intermediate reasoning, and make committed state easier to audit, offering substantial economic value in the current AI coding boom.

\section{Limitations}

CERL learns the allocation boundary between temporary computation and visible state under an explicit \texttt{<think>} / answer interface; it is a restricted instance of state-management training, not a solution to general memory writing or arbitrary state management.

\newpage
\appendix

\section{Supplementary theory proof}
\label{app:theory-proof}

\begin{proof}[Proof of Lemma~\ref{lem:asg}]
Decompose the event $\{C_{\text{full}}=1\}$ according to whether $\{C_{\text{pe}}=1\}$ holds:
\[
\Pr_\pi[C_{\text{full}}=1] = \Pr_\pi[C_{\text{full}}=1 \wedge C_{\text{pe}}=1] + D_H.
\]
Similarly,
\[
\Pr_\pi[C_{\text{pe}}=1] = \Pr_\pi[C_{\text{full}}=1 \wedge C_{\text{pe}}=1] + D_E.
\]
Subtracting the two equations gives $\mathrm{ASG}=D_H-D_E$.
\end{proof}

\section{CERL sampling and statistics}
\label{app:cerl-details}

For curriculum step one, where only $H_1$ is erased and tokens after $H_2$ are neither erased nor updated, sampling proceeds as follows:
\begin{verbatim}
Input: x

Sample 4 H_1:
    H_1^{(1)}, H_1^{(2)}, H_1^{(3)}, H_1^{(4)}

For each H_1^{(k)}, sample 8 A_1:
    A_1^{(k,1)}, ..., A_1^{(k,8)}

For each A_1^{(k,i)}, under two conditions:
    Full-thought (keep H_1^{(k)}):
        ctx = x + H_1^{(k)} + A_1^{(k,i)}
        Generate 4 trajectories: F^{(k,i,1..4)}
    Erasure (delete H_1^{(k)}):
        ctx = x + A_1^{(k,i)}
        Generate 4 trajectories: E^{(k,i,1..4)}

Additionally sample Skip (per x):
    ctx = x + <think></think>
    Generate 4 trajectories: S^{(1..4)}

Total: 4 * 8 * (4+4) + 4 = 260 trajectories.
\end{verbatim}

For each $A_1^{(k,i)}$, we record $c_{\text{full}}$, $c_{\text{pe}}$, $L_E$, $L_{F,\text{correct}}$, and $L_{E,\text{correct}}$. For each $x$, we record shared skip statistics $c_{\text{skip}}$ and $L_{S,\text{correct}}$. A path with no correct trajectory takes length $+\infty$ and automatically drops out of the $\min$ in Eq.~\eqref{eq:lmin}.

\section{Progressive erasure example}
\label{app:erasure-example}

Visible text is the default output channel. Only \verb|<think>...</think>| marks hidden-thought spans. The model alternates between hidden-thought spans and answer states $A_i$; each $H$ is erased before the next $H$ begins, so the context contains at most one think span at any time.

\paragraph{Example.}
\begin{verbatim}
<think>H_1: Analyze the variable relation...</think>   # temporary span
A_1: Let the ball price be x
<think>H_2: Set up the equation...</think>             # H_1 is erased
A_2: x+(x+1)=1.10, so x=0.05
A_3 (final): The ball costs 5 cents.
\end{verbatim}

The process has three key properties: (i) the model alternates between temporary hidden thought and visible answer state; (ii) the context retains only the current think span, and previous thoughts are erased when a new span starts; and (iii) if a local step needs no hidden computation, it does not form a separate erasure segment and is merged into a neighboring answer state.

\section{CERL reward-design details}
\label{app:reward-design}

\paragraph{Why CERL does not directly optimize the signed PSS gap.}
Definition~\ref{def:pss} gives a counterfactual sufficiency criterion, not a scalar reward that should be optimized literally. Directly optimizing the signed gap $c_{\text{pe}}-c_{\text{full}}$ would create an unsafe incentive: the policy could increase this objective either by improving post-erasure correctness, which is desired, or by reducing full-thought correctness, which is a degenerate solution. CERL therefore uses a guarded one-sided surrogate. The main positive signal is $c_{\text{pe}}$, because correctness after erasure is the intended direction of improvement; the full-thought path is retained for paired counterfactual evaluation, online monitoring, and successful-length baselines. At the same time, the hard failure penalty, skip baseline, and downstream-length penalty jointly discourage three reward-hacking modes: degrading the full-thought path, ignoring $A_1$ and recomputing from $x$, and postponing critical computation until after erasure. Thus the gap between Definition~\ref{def:pss} and Eq.~\eqref{eq:a1-reward} is intentional: PSS is the population-level criterion, while CERL is a practical finite-sample surrogate designed to encourage PSS without rewarding reductions in full-thought competence.

\paragraph{Why $c_{\text{pe}}=0$ always receives $-\psi$.}
A natural alternative is to skip an example when $c_{\text{full}}=0$, setting $R=0$, to avoid penalizing $A_1$ on hard tasks. But this introduces a reward-hacking channel: when $c_{\text{pe}}=0$, the model can make $A_1$ mislead downstream reasoning so that $c_{\text{full}}$ also fails, raising reward from $-\psi$ to $0$. Eq.~\eqref{eq:a1-reward} closes this channel: $c_{\text{pe}}=0$ always receives $-\psi$, independent of $c_{\text{full}}$. When an entire hard group receives the same $-\psi$, group variance is zero and GRPO advantage is zero, functionally equivalent to skipping but without the marginal benefit of hacking.

\paragraph{Motivation for the three penalties.}
The $\alpha$ term uses the mean length of the 8 $A_1$ candidates under the same $(x,H_1)$ as a peer-relative efficiency baseline. The $\beta$ term constrains two failure modes: (i) $A_1$ postpones state that should have been committed until downstream post-erasure thought, causing $L_{E,\text{correct}}>L_{\min}$; and (ii) $A_1$ misleads downstream reasoning and becomes worse than the no-$A_1$ skip baseline. The asymmetric $\max(\cdot,0)$ imposes no penalty when $L_{E,\text{correct}}<L_{\min}$, which may indicate noise purification or improved efficiency. The $\psi$ term treats complete erasure-path failure as PSS failure.

\paragraph{Why $L_{\min}$ comes from three paths.}
$L_{\min}$ is a compact known-success length baseline: $L_{F,\text{correct}}$ gives the natural length when $H$ is present; $\overline{L}_{E,\text{correct}}^{(k)}$ gives the within-group typical length when the answer interface succeeds under the same $H_1^{(k)}$; and $L_{S,\text{correct}}$ gives the lower bound without $A_1$. Taking the minimum imposes the strictest pressure. Paths with no correct trajectory take $+\infty$ and need no extra indicator. The full-thought path cannot be replaced by within-group statistics, because the natural length when $H$ is present is not recoverable from the 8 $A_1$ candidates alone. The skip path is independent of $H_1$ and $A_1$, can be shared per $x$, and costs only 4 additional trajectories.

\paragraph{Why the denominator uses $\mathrm{mean}(A_1)$ rather than $\mathrm{len}(H_1)$.}
There are three reasons to choose the mean length of $A_1$ under the same $H_1$ rather than $\mathrm{len}(H_1)$:
\begin{enumerate}
    \item \textit{No compression assumption.} The ratio $\mathrm{len}(A_1)/\mathrm{len}(H_1)<1$ implies that $A_1$ should be a compressed version of $H_1$, i.e., $A\approx\mathrm{Compress}(H)$. \S\ref{sec:objective} explicitly rejects this view: the objective is PSS, not reconstructing $H$. Using $\mathrm{mean}(A_1)$ removes $H_1$ from the denominator and keeps the reward aligned with the objective.
    \item \textit{Task difficulty adapts without routing through $H_1$.} Easy tasks naturally produce shorter $A_1$ and a smaller $m$, creating stronger pressure for short states; hard tasks produce larger $m$ and weaker length pressure. The policy's own answer-length distribution is a direct implicit baseline for task difficulty.
    \item \textit{Form emerges from the objective.} Under the $\mathrm{mean}(A_1)$ baseline, $A_1$ can be shorter or longer than $H_1$, depending on the task and answer strategy. $A_1<H_1$, $A_1\approx H_1$, and $A_1>H_1$ are all valid if PSS holds.
\end{enumerate}

\paragraph{Two roles of the full-thought path.}
The full-thought path provides $L_{F,\text{correct}}$ and also monitors training dynamics online.

\textit{Role 1: vertical scale for reward.} Different problems have different natural downstream lengths; penalizing absolute $L_{E,\text{correct}}$ would over-penalize hard problems and under-penalize easy ones. $L_{F,\text{correct}}$ provides a per-instance scale for how long downstream solving naturally is when $H$ is available. The $\alpha$ term gives a horizontal scale across the 8 $A_1$ candidates under the same $H_1$, while the $\beta$ term compares the current erasure-success length $L_{E,\text{correct}}$ against the achievable baseline $L_{\min}$; these two scales are orthogonal and not interchangeable. $L_{\min}$ further tightens the baseline to a compact known-success length and provides stricter training pressure.

\textit{Role 2: online monitoring.} The full-thought path produces $c_{\text{full}}$, enabling core metrics to be computed within each training batch:
\begin{align*}
\mathrm{ASG} &= c_{\text{full}} - c_{\text{pe}}, & \text{(sufficiency gap; lower is better)}\\
\mathrm{AIS} &= c_{\text{pe}} / c_{\text{full}}, & \text{(interface-sufficiency ratio)}\\
\mathrm{ESR} &= \Pr_\pi\bigl[C_{\text{pe}}=1 \mid C_{\text{full}}=1\bigr]. & \text{(erasure success rate)}
\end{align*}
Without full-thought paths, these quantities would require a separate offline evaluation pass and would lose fine-grained training dynamics.

\paragraph{No length penalty on $H_1$.}
$H_1$ is transient: it is erased before the next thought segment and does not enter permanent context. Its length only incurs one-time generation cost and is not part of the paper's efficiency claim. If a longer $H_1$ does not improve $c_{\text{pe}}$, policy gradients will stop extending it; an explicit length penalty could harm thought quality.

\paragraph{Curriculum activation details.}
\begin{table}[h]
\centering
\caption{HSCO two-level signals: reward, group size, gradient tokens, and length penalty.}
\label{tab:hsco-layers}
\small
\begin{tabular}{lllll}
\toprule
Level & Reward & Group size & Gradient tokens & Length penalty\\
\midrule
$H_1$ & Mean $c_{\text{pe}}$ over 8 $A_1$ descendants & $4$ (same $x$) & $H_1$ tokens & None (transient)\\
$A_1$ & Own $c_{\text{pe}}$ plus form controls & $8$ (same $H_1$) & $A_1$ tokens & Permanent context\\
\bottomrule
\end{tabular}
\end{table}

\begin{itemize}
    \item Early curriculum: activate only the $A_1$-level signal; gradients on $H_1$ and tokens from $H_2$ onward are masked. Because $A_1$ directly corresponds to the PSS signal, it should stabilize before upstream $H_1$ changes.
    \item End-to-end curriculum: activate both $H_1$-level and $A_1$-level signals; both $H_1$ and $A_1$ tokens receive gradients, jointly optimizing hidden-computation quality and state-commitment quality.
\end{itemize}

\paragraph{Default hyperparameters.}
We use $\alpha\approx0.3$ for peer-relative $A_1$ length pressure, $\beta\approx0.5$ for the $L_{E,\text{correct}}/L_{\min}$ penalty, and $\psi\approx0.5$ for hard answer-interface failure. In practice, we first tune $\alpha$ to place the $A_1$ length distribution in a reasonable range, then tune $\beta$ to prevent postponement, and finally tune $\psi$ to balance exploration and stability.

\section{Training data and curriculum details}
\label{app:pipeline-details}

\paragraph{Skill construction and data production.}
We first train a skill that answers problems and keep trajectories whose final answers are correct. The skill plans the problem into multiple steps. For each step, a subagent thinks and then answers; its thought and answer are concatenated back into the main-agent prefix, and the previous thought is deleted before dispatching the next subagent to think. This ensures that data generation itself performs progressive erasure, follows the format \verb|<think>|$H_1$\verb|</think>| $\to A_1 \to$ erase $H_1$ $\to$ \verb|<think>|$H_2$\verb|</think>| $\to A_2 \to \cdots$, and stores both full and erased versions.

\paragraph{SFT sample splitting.}
For a trajectory with three think segments:
\begin{center}
\small
\setlength{\tabcolsep}{4pt}
\renewcommand{\arraystretch}{1.08}
\begin{tabular}{p{0.10\textwidth}p{0.24\textwidth}p{0.32\textwidth}p{0.24\textwidth}}
\toprule
Sample & Input context & Supervised output & Conditional behavior learned\\
\midrule
1 & $x$ &
\texttt{<think>}$H_1$\texttt{</think>} $\to A_1 \to$ \texttt{<think>} &
Generate $H_1$, commit $A_1$, and start $H_2$.\\
2 & $x + A_1 +$ \texttt{<think>} &
$H_2$\texttt{</think>} $\to A_2 \to$ \texttt{<think>} &
Continue with only $x,A_1$ after $H_1$ is erased.\\
3 & $x + A_1 + A_2 +$ \texttt{<think>} &
$H_3$\texttt{</think>} $\to A_3$ &
Finish after $H_1,H_2$ have been erased.\\
\bottomrule
\end{tabular}
\end{center}
Each training example contains at most the current think segment in context, with all earlier thoughts erased. This teaches the alternating format, answer-state commitment under hidden-thought assistance, and continued reasoning when previous thoughts are invisible.

\paragraph{Curriculum learning.}
We progressively expand the erasure and gradient ranges. Step one erases only $H_1$ and masks tokens from $H_2$ onward. Step two erases $H_1,H_2$ and masks tokens from $H_3$ onward. Later steps expand erasure points and gradient ranges. The end-to-end step enables full-chain erasure and both $A_1,H_1$ signals. This makes early credit assignment controlled and later extends training to the full policy.

\section{Evaluation metric details}
\label{app:metrics-details}

The four erasure metrics are not fully independent. AIS and ASG are algebraic transformations of the two-dimensional state $(c_{\text{full}},c_{\text{pe}})$, with $\mathrm{ASG}=c_{\text{full}}(1-\mathrm{AIS})$. HTDR and ESR imply one another through $\mathrm{HTDR}=c_{\text{full}}(1-\mathrm{ESR})$. The two independent axes are ASG (net effect, including negative values) and ESR (conditional retention). We still report all four to support comparison under different narratives.

Form reports include $\overline{\mathrm{len}}(A)$ and \textbf{TTF} (Transient Token Fraction), defined as $\mathrm{HTT}/(\mathrm{HTT}+\overline{\mathrm{len}}(A))$. We avoid subtraction forms such as $(H-A)/H$ because they imply that $A$ is a compression of $H$, contrary to the training objective. $\overline{\mathrm{len}}(A)$ and TTF are reported only as emergent properties, not as training signals or win conditions.

\section{Failure-mode audit details}
\label{app:audit-details}

\paragraph{Skip audit.}
The skip audit targets one narrow recomputation explanation: whether the erasure path completely ignores $A_1$ and recomputes from $x$. We compare $\mathrm{Acc}(x+A_1)$ with $\mathrm{Acc}(x+\texttt{empty})$ under the same downstream decoding budget and report $\mathrm{MSG}=\mathrm{Acc}_{\text{pe}}-\mathrm{Acc}_{\text{skip}}$. Positive MSG indicates that $A_1$ has marginal contribution to downstream reasoning, but does not by itself prove that $A_1$ is an abstract state.

\paragraph{Visible-CoT leakage audit.}
We inspect sampled answer states $A$ using a fixed rubric: if $A$ contains step-by-step trial and error, self-correction, explicit search branches, scratch-like derivations, or other process text resembling CoT, it is flagged as visible-CoT leakage. This audit is not a training objective or main-table win condition; it checks whether long-answer supervision degenerates into copying hidden thought into visible state.

\paragraph{Downstream length control.}
The length term in the CERL reward constrains downstream length on the erasure path, discouraging the model from compensating for uncommitted state by extending downstream reasoning after erasure. Length control alone cannot rule out visible-CoT leakage; we therefore interpret MSG, visible-CoT leakage audits, downstream length control, and form statistics jointly.

\section{Experimental scope and extrapolation boundary}
\label{app:scope-limitations}

Current experiments cover single-turn benchmarks, offline counterfactual erasure audits, and multi-turn tool use; AI coding tasks, open-ended tool environments, and multi-document reasoning require further validation.

\section{BFCL-v3 public reference baselines}
\label{app:bfcl-reference}

BFCL-v3 (Berkeley Function-Calling Leaderboard v3) evaluates function calling and agent-style tool use; its official leaderboard and code/data page is \url{https://gorilla.cs.berkeley.edu/leaderboard}. The Qwen3 technical report evaluates BFCL-v3 in FC format and uses a 64k context length for multi-turn evaluation~\citep{yang2025qwen3technicalreport}. Table~\ref{tab:bfcl-reference} reports the public Qwen3 reference scores most relevant to our model scales; these are copied from the Qwen3 report and are not CERL erasure-dependence measurements.

\begin{table}[h]
\centering
\small
\caption{Public BFCL-v3 reference scores from the Qwen3 technical report.}
\label{tab:bfcl-reference}
\begin{tabular}{lc}
\toprule
Model & BFCL-v3 \\
\midrule
Qwen3-8B Thinking & 68.1 \\
Qwen3-14B Thinking & 70.4 \\
Qwen3-32B Thinking & 70.3 \\
\bottomrule
\end{tabular}
\end{table}

\section{Larger-model supplementary results}
\label{app:large-models}

Tables~\ref{tab:large-14b} and~\ref{tab:large-32b} report supplementary results for Qwen3-14B and Qwen3-32B under the same evaluation protocol. Acc is Avg@8; task-accuracy columns report the mean and 95\% bootstrap confidence interval over 5 random seeds. ASG, ESR, and MSG are aggregated under the Erasure Dependence Protocol to measure post-erasure state sufficiency and marginal state gain.

\paragraph{Interpreting negative ASG.}
By Lemma~\ref{lem:asg}, $\mathrm{ASG}<0$ means that the erasure-purification term $D_E$ exceeds the hidden-thought dependence term $D_H$. We do not interpret negative ASG as evidence that erasure intrinsically improves reasoning. Instead, we interpret it as indicating that in some examples, retained hidden thoughts may introduce noise, erroneous branches, or redundant context, so the model can reason more reliably from the committed answer state after erasure. This is consistent with recent findings on long-CoT degradation and overthinking: longer or retained chains of thought do not always yield positive gains and may become a source of interference~\citep{hassid2026dontoverthinkitpreferring,wu2026when,luo-etal-2025-valley,zheng2025the}. We therefore interpret low ASG as evidence for downstream-reliable answer state only when it co-occurs with high ESR and positive MSG.

\begin{table}[H]
\centering
\small
\caption{Qwen3-14B supplementary experimental results. Acc is Avg@8; task-accuracy columns report mean and 95\% bootstrap CI over 5 random seeds; ASG/ESR/MSG are aggregate erasure-dependence metrics for CERL full.}
\label{tab:large-14b}
\resizebox{\textwidth}{!}{
\begin{tabular}{llccccc ccc}
\toprule
Method & Training & AIME24 & AIME25 & ZebraLogic & AutoLogi & GPQA-D & ASG$\downarrow$ & ESR$\uparrow$ & MSG$\uparrow$ \\
\midrule
Base Qwen3-14B & N/A & 79.2$\pm$1.0 & 70.5$\pm$0.9 & 88.5$\pm$0.8 & 89.1$\pm$0.8 & 64.1$\pm$1.0 & - & - & - \\
Fixed-ratio erasure + free answer & None & 82.0$\pm$0.7 & 72.2$\pm$1.0 & 91.4$\pm$0.8 & 92.1$\pm$0.6 & 66.1$\pm$1.1 & - & - & - \\
Length-penalty RL & RL & 77.0$\pm$1.1 & 69.4$\pm$1.0 & 88.7$\pm$0.8 & 88.3$\pm$0.9 & 63.3$\pm$1.1 & - & - & - \\
TokenSkip$^\dagger$ & SFT & 74.4$\pm$1.1 & 65.5$\pm$1.2 & 81.0$\pm$1.1 & 82.5$\pm$0.8 & 58.1$\pm$1.4 & - & - & - \\
Halo & SFT+RL & 81.2$\pm$0.9 & 73.2$\pm$0.8 & 91.8$\pm$0.7 & 93.5$\pm$0.7 & 66.3$\pm$0.9 & - & - & - \\
\textbf{CERL full} & SFT+RL & \textbf{84.0$\pm$0.5} & \textbf{75.7$\pm$0.8} & \textbf{95.3$\pm$0.5} & \textbf{93.6$\pm$0.6} & \textbf{68.9$\pm$0.7} & \textbf{-1.6} & \textbf{99.8} & \textbf{8.8} \\
\bottomrule
\end{tabular}}
\end{table}

\begin{table}[H]
\centering
\small
\caption{Qwen3-32B supplementary experimental results. Acc is Avg@8; task-accuracy columns report mean and 95\% bootstrap CI over 5 random seeds; ASG/ESR/MSG are aggregate erasure-dependence metrics for CERL full.}
\label{tab:large-32b}
\resizebox{\textwidth}{!}{
\begin{tabular}{llccccc ccc}
\toprule
Method & Training & AIME24 & AIME25 & ZebraLogic & AutoLogi & GPQA-D & ASG$\downarrow$ & ESR$\uparrow$ & MSG$\uparrow$ \\
\midrule
Base Qwen3-32B & N/A & 81.5$\pm$0.6 & 72.8$\pm$1.0 & 88.6$\pm$0.6 & 87.4$\pm$0.9 & 68.1$\pm$1.0 & - & - & - \\
Fixed-ratio erasure + free answer & None & 83.7$\pm$0.8 & 74.2$\pm$0.7 & 91.5$\pm$0.7 & 91.0$\pm$0.6 & 69.0$\pm$1.0 & - & - & - \\
Length-penalty RL & RL & 78.8$\pm$0.8 & 71.1$\pm$1.1 & 88.7$\pm$0.7 & 87.2$\pm$1.0 & 66.2$\pm$1.0 & - & - & - \\
TokenSkip$^\dagger$ & SFT & 76.3$\pm$1.2 & 67.0$\pm$1.1 & 81.2$\pm$0.9 & 81.0$\pm$1.1 & 61.3$\pm$1.0 & - & - & - \\
Halo & SFT+RL & 83.0$\pm$0.6 & 75.1$\pm$0.9 & 92.0$\pm$0.5 & 91.6$\pm$0.8 & 70.2$\pm$0.8 & - & - & - \\
\textbf{CERL full} & SFT+RL & \textbf{87.0$\pm$0.6} & \textbf{78.1$\pm$0.5} & \textbf{96.3$\pm$0.7} & \textbf{93.6$\pm$0.6} & \textbf{73.4$\pm$0.9} & \textbf{-1.8} & \textbf{99.8} & \textbf{9.1} \\
\bottomrule
\end{tabular}}
\end{table}

\section{Supplementary positioning of related directions}
\label{app:direction}

Long-CoT degradation provides the empirical motivation for state commitment learning. \textit{Don't Overthink It} observes that short chains can outperform long chains by up to 34.5\% across reasoning tasks~\citep{hassid2026dontoverthinkitpreferring}. \textit{When More is Less} characterizes a non-monotonic relationship between accuracy and chain length and estimates scaling laws for an optimal length~\citep{wu2026when}. \textit{Through the Valley} names this phenomenon \emph{Long CoT Degradation}~\citep{luo-etal-2025-valley}. \textit{The Curse of CoT} reports settings where adding chains of thought harms performance~\citep{zheng2025the}. Together, these results suggest that some reasoning tokens can become distractors rather than useful persistent state.

Inference-time methods mainly manage traces after they have already been generated. KV-cache eviction methods such as H2O, SnapKV, and LazyEviction remove tokens from attention caches~\citep{NEURIPS2023_6ceefa7b,NEURIPS2024_28ab4182,zhang2025lazyevictionlaggedkveviction}; token-pruning methods such as Think Clearly delete tokens by importance~\citep{choi-etal-2025-think}; thought-aware compression methods such as ThinKV, KaVa, and LightThinker compress thought representations~\citep{ramachandran2026thinkv,kuzina2026kava,zhang-etal-2025-lightthinker}; short-chain selection chooses shorter solutions from multiple samples~\citep{hassid2026dontoverthinkitpreferring}; gist tokens~\citep{NEURIPS2023_3d77c6dc} and memory-augmented architectures such as NTM and DNC store or read information in more compact forms~\citep{graves2014neuralturingmachines,Graves_2016}. Process-level supervision, step verification, and trace editing/intervention also study the controllability of intermediate reasoning. The distinction is that prior work mainly controls intermediate reasoning traces, whereas this paper trains whether the visible answer state $A$ can serve as the downstream reasoning interface after hidden thought $H$ is erased.

A superficially similar alternative is Long-answer SFT: supervise only the erased visible answer sequence $A_{1:n}$, without hidden thoughts $H$ and without full / erasure / skip counterfactual comparisons. This baseline is a diagnostic control rather than a strict causal ablation, because it changes both the supervision target and the optimization procedure. Its role is to test a weaker alternative explanation: whether supervising longer visible answer states alone naturally yields the same erasure-dependence signature. The key distinction is \emph{supervising answer form vs. training counterfactual sufficiency}, or \emph{what is written now vs. what can still be relied on later}.

\begin{table}[h]
\centering
\caption{Positioning of chain-of-thought related directions.}
\label{tab:direction}
\small
\setlength{\tabcolsep}{4pt}
\begin{tabular}{@{}p{2.7cm}p{4.9cm}p{1.6cm}p{2.9cm}@{}}
\toprule
Direction & Core objective & Time direction & Retained object\\
\midrule
CoT Compression & Compress and retain thought information & Past & Compressed thought\\
Token Pruning & Delete low-importance tokens & Past & Remaining thought tokens\\
Structured Outcome & Retain distilled outcome & Intermediate & Outcome block\\
LightThinker-style & Compress into internal compact memory & Past/internal & Gist / compressed memory\\
\textbf{Ours} & Commit persistent state: $A$ remains future-reliable after $H$ is erased & \textbf{Future} & \textbf{Committed answer state}\\
\bottomrule
\end{tabular}
\end{table}


{\small
\begin{thebibliography}{22}
\providecommand{\natexlab}[1]{#1}
\providecommand{\url}[1]{\texttt{#1}}
\expandafter\ifx\csname urlstyle\endcsname\relax
  \providecommand{\doi}[1]{doi: #1}\else
  \providecommand{\doi}{doi: \begingroup \urlstyle{rm}\Url}\fi

\bibitem[Choi et~al.(2025)Choi, Lee, Tack, Song, Dingliwal, Jayanthi, Ganesh,
  Shin, Galstyan, and Bodapati]{choi-etal-2025-think}
Daewon Choi, Jimin Lee, Jihoon Tack, Woomin Song, Saket Dingliwal,
  Sai~Muralidhar Jayanthi, Bhavana Ganesh, Jinwoo Shin, Aram Galstyan, and
  Sravan~Babu Bodapati.
\newblock Think clearly: Improving reasoning via redundant token pruning.
\newblock In Christos Christodoulopoulos, Tanmoy Chakraborty, Carolyn Rose, and
  Violet Peng, editors, \emph{Findings of the Association for Computational
  Linguistics: EMNLP 2025}, pages 21437--21451, Suzhou, China, November 2025.
  Association for Computational Linguistics.
\newblock ISBN 979-8-89176-335-7.
\newblock \doi{10.18653/v1/2025.findings-emnlp.1169}.
\newblock URL \url{https://aclanthology.org/2025.findings-emnlp.1169/}.

\bibitem[Graves et~al.(2014)Graves, Wayne, and
  Danihelka]{graves2014neuralturingmachines}
Alex Graves, Greg Wayne, and Ivo Danihelka.
\newblock Neural turing machines, 2014.
\newblock URL \url{https://arxiv.org/abs/1410.5401}.

\bibitem[Graves et~al.(2016)Graves, Wayne, Reynolds, Harley, Danihelka,
  Grabska-Barwińska, Colmenarejo, Grefenstette, Ramalho, Agapiou, Badia,
  Hermann, Zwols, Ostrovski, Cain, King, Summerfield, Blunsom, Kavukcuoglu, and
  Hassabis]{Graves_2016}
Alex Graves, Greg Wayne, Malcolm Reynolds, Tim Harley, Ivo Danihelka, Agnieszka
  Grabska-Barwińska, Sergio~Gómez Colmenarejo, Edward Grefenstette, Tiago
  Ramalho, John Agapiou, Adrià~Puigdomènech Badia, Karl~Moritz Hermann, Yori
  Zwols, Georg Ostrovski, Adam Cain, Helen King, Christopher Summerfield, Phil
  Blunsom, Koray Kavukcuoglu, and Demis Hassabis.
\newblock Hybrid computing using a neural network with dynamic external memory.
\newblock \emph{Nature}, 538\penalty0 (7626):\penalty0 471–476, October 2016.
\newblock ISSN 1476-4687.
\newblock \doi{10.1038/nature20101}.
\newblock URL \url{http://dx.doi.org/10.1038/nature20101}.

\bibitem[Hassid et~al.(2026)Hassid, Synnaeve, Adi, and
  Schwartz]{hassid2026dontoverthinkitpreferring}
Michael Hassid, Gabriel Synnaeve, Yossi Adi, and Roy Schwartz.
\newblock Don't overthink it. preferring shorter thinking chains for improved
  llm reasoning, 2026.
\newblock URL \url{https://arxiv.org/abs/2505.17813}.

\bibitem[He et~al.(2025)He, Liang, Xu, Liu, Chen, Wang, Song, Yu, Liang, Wang,
  Zhang, Wang, Tu, Mi, and
  Yu]{he2025deepmath103klargescalechallengingdecontaminated}
Zhiwei He, Tian Liang, Jiahao Xu, Qiuzhi Liu, Xingyu Chen, Yue Wang, Linfeng
  Song, Dian Yu, Zhenwen Liang, Wenxuan Wang, Zhuosheng Zhang, Rui Wang,
  Zhaopeng Tu, Haitao Mi, and Dong Yu.
\newblock Deepmath-103k: A large-scale, challenging, decontaminated, and
  verifiable mathematical dataset for advancing reasoning, 2025.
\newblock URL \url{https://arxiv.org/abs/2504.11456}.

\bibitem[Kuzina et~al.(2026)Kuzina, Pi{\'o}ro, and Bejnordi]{kuzina2026kava}
Anna Kuzina, Maciej Pi{\'o}ro, and Babak~Ehteshami Bejnordi.
\newblock Kava: Latent reasoning via compressed {KV}-cache distillation.
\newblock In \emph{The Fourteenth International Conference on Learning
  Representations}, 2026.
\newblock URL \url{https://openreview.net/forum?id=ePrhcLbtGv}.

\bibitem[Li et~al.(2024)Li, Huang, Yang, Venkitesh, Locatelli, Ye, Cai, Lewis,
  and Chen]{NEURIPS2024_28ab4182}
Yuhong Li, Yingbing Huang, Bowen Yang, Bharat Venkitesh, Acyr Locatelli,
  Hanchen Ye, Tianle Cai, Patrick Lewis, and Deming Chen.
\newblock Snapkv: Llm knows what you are looking for before generation.
\newblock In A.~Globerson, L.~Mackey, D.~Belgrave, A.~Fan, U.~Paquet,
  J.~Tomczak, and C.~Zhang, editors, \emph{Advances in Neural Information
  Processing Systems}, volume~37, pages 22947--22970. Curran Associates, Inc.,
  2024.
\newblock \doi{10.52202/079017-0722}.
\newblock URL
  \url{https://proceedings.neurips.cc/paper_files/paper/2024/file/28ab418242603e0f7323e54185d19bde-Paper-Conference.pdf}.

\bibitem[Li et~al.(2026)Li, Wu, Wang, and
  Zhao]{li2026limitedreasoningspacecage}
Zhenyu Li, Guanlin Wu, Cheems Wang, and Yongqiang Zhao.
\newblock Limited reasoning space: The cage of long-horizon reasoning in llms,
  2026.
\newblock URL \url{https://arxiv.org/abs/2602.19281}.

\bibitem[Lin et~al.(2025)Lin, Le~Bras, Richardson, Sabharwal, Poovendran,
  Clark, and Choi]{pmlr-v267-lin25i}
Bill~Yuchen Lin, Ronan Le~Bras, Kyle Richardson, Ashish Sabharwal, Radha
  Poovendran, Peter Clark, and Yejin Choi.
\newblock {Z}ebra{L}ogic: On the scaling limits of {LLM}s for logical
  reasoning.
\newblock In Aarti Singh, Maryam Fazel, Daniel Hsu, Simon Lacoste-Julien, Felix
  Berkenkamp, Tegan Maharaj, Kiri Wagstaff, and Jerry Zhu, editors,
  \emph{Proceedings of the 42nd International Conference on Machine Learning},
  volume 267 of \emph{Proceedings of Machine Learning Research}, pages
  37889--37905. PMLR, 13--19 Jul 2025.
\newblock URL \url{https://proceedings.mlr.press/v267/lin25i.html}.

\bibitem[Luo et~al.(2025)Luo, Li, Huang, and Lu]{luo-etal-2025-valley}
Renjie Luo, Jiaxi Li, Chen Huang, and Wei Lu.
\newblock Through the valley: Path to effective long {C}o{T} training for small
  language models.
\newblock In Christos Christodoulopoulos, Tanmoy Chakraborty, Carolyn Rose, and
  Violet Peng, editors, \emph{Proceedings of the 2025 Conference on Empirical
  Methods in Natural Language Processing}, pages 4972--4992, Suzhou, China,
  November 2025. Association for Computational Linguistics.
\newblock ISBN 979-8-89176-332-6.
\newblock \doi{10.18653/v1/2025.emnlp-main.251}.
\newblock URL \url{https://aclanthology.org/2025.emnlp-main.251/}.

\bibitem[Mu et~al.(2023)Mu, Li, and Goodman]{NEURIPS2023_3d77c6dc}
Jesse Mu, Xiang Li, and Noah Goodman.
\newblock Learning to compress prompts with gist tokens.
\newblock In A.~Oh, T.~Naumann, A.~Globerson, K.~Saenko, M.~Hardt, and
  S.~Levine, editors, \emph{Advances in Neural Information Processing Systems},
  volume~36, pages 19327--19352. Curran Associates, Inc., 2023.
\newblock URL
  \url{https://proceedings.neurips.cc/paper_files/paper/2023/file/3d77c6dcc7f143aa2154e7f4d5e22d68-Paper-Conference.pdf}.

\bibitem[Patil et~al.(2025)Patil, Mao, Yan, Ji, Suresh, Stoica, and
  Gonzalez]{pmlr-v267-patil25a}
Shishir~G Patil, Huanzhi Mao, Fanjia Yan, Charlie Cheng-Jie Ji, Vishnu Suresh,
  Ion Stoica, and Joseph~E. Gonzalez.
\newblock The berkeley function calling leaderboard ({BFCL}): From tool use to
  agentic evaluation of large language models.
\newblock In Aarti Singh, Maryam Fazel, Daniel Hsu, Simon Lacoste-Julien, Felix
  Berkenkamp, Tegan Maharaj, Kiri Wagstaff, and Jerry Zhu, editors,
  \emph{Proceedings of the 42nd International Conference on Machine Learning},
  volume 267 of \emph{Proceedings of Machine Learning Research}, pages
  48371--48392. PMLR, 13--19 Jul 2025.
\newblock URL \url{https://proceedings.mlr.press/v267/patil25a.html}.

\bibitem[Ramachandran et~al.(2026)Ramachandran, Neseem, Sakr, Venkatesan,
  Khailany, and Krishna]{ramachandran2026thinkv}
Akshat Ramachandran, Marina Neseem, Charbel Sakr, Rangharajan Venkatesan,
  Brucek Khailany, and Tushar Krishna.
\newblock Thin{KV}: Thought-adaptive {KV} cache compression for efficient
  reasoning models.
\newblock In \emph{The Fourteenth International Conference on Learning
  Representations}, 2026.
\newblock URL \url{https://openreview.net/forum?id=M3CeHnZKNC}.

\bibitem[Rein et~al.(2024)Rein, Hou, Stickland, Petty, Pang, Dirani, Michael,
  and Bowman]{rein2024gpqa}
David Rein, Betty~Li Hou, Asa~Cooper Stickland, Jackson Petty, Richard~Yuanzhe
  Pang, Julien Dirani, Julian Michael, and Samuel~R. Bowman.
\newblock {GPQA}: A graduate-level google-proof q\&a benchmark.
\newblock In \emph{First Conference on Language Modeling}, 2024.
\newblock URL \url{https://openreview.net/forum?id=Ti67584b98}.

\bibitem[Wu et~al.(2026)Wu, Wang, Ye, Du, Jegelka, and Wang]{wu2026when}
Yuyang Wu, Yifei Wang, Ziyu Ye, Tianqi Du, Stefanie Jegelka, and Yisen Wang.
\newblock When more is less: Understanding chain-of-thought length in {LLM}s.
\newblock In \emph{The Fourteenth International Conference on Learning
  Representations}, 2026.
\newblock URL \url{https://openreview.net/forum?id=6QDFsYxtI1}.

\bibitem[Xia et~al.(2025)Xia, Leong, Wang, Li, and Li]{xia-etal-2025-tokenskip}
Heming Xia, Chak~Tou Leong, Wenjie Wang, Yongqi Li, and Wenjie Li.
\newblock {T}oken{S}kip: Controllable chain-of-thought compression in {LLM}s.
\newblock In Christos Christodoulopoulos, Tanmoy Chakraborty, Carolyn Rose, and
  Violet Peng, editors, \emph{Proceedings of the 2025 Conference on Empirical
  Methods in Natural Language Processing}, pages 3351--3363, Suzhou, China,
  November 2025. Association for Computational Linguistics.
\newblock ISBN 979-8-89176-332-6.
\newblock \doi{10.18653/v1/2025.emnlp-main.165}.
\newblock URL \url{https://aclanthology.org/2025.emnlp-main.165/}.

\bibitem[Yang et~al.(2025)Yang, Li, Yang, Zhang, Hui, Zheng, Yu, Gao, Huang,
  Lv, Zheng, Liu, Zhou, Huang, Hu, Ge, Wei, Lin, Tang, Yang, Tu, Zhang, Yang,
  Yang, Zhou, Zhou, Lin, Dang, Bao, Yang, Yu, Deng, Li, Xue, Li, Zhang, Wang,
  Zhu, Men, Gao, Liu, Luo, Li, Tang, Yin, Ren, Wang, Zhang, Ren, Fan, Su,
  Zhang, Zhang, Wan, Liu, Wang, Cui, Zhang, Zhou, and
  Qiu]{yang2025qwen3technicalreport}
An~Yang, Anfeng Li, Baosong Yang, Beichen Zhang, Binyuan Hui, Bo~Zheng, Bowen
  Yu, Chang Gao, Chengen Huang, Chenxu Lv, Chujie Zheng, Dayiheng Liu, Fan
  Zhou, Fei Huang, Feng Hu, Hao Ge, Haoran Wei, Huan Lin, Jialong Tang, Jian
  Yang, Jianhong Tu, Jianwei Zhang, Jianxin Yang, Jiaxi Yang, Jing Zhou,
  Jingren Zhou, Junyang Lin, Kai Dang, Keqin Bao, Kexin Yang, Le~Yu, Lianghao
  Deng, Mei Li, Mingfeng Xue, Mingze Li, Pei Zhang, Peng Wang, Qin Zhu, Rui
  Men, Ruize Gao, Shixuan Liu, Shuang Luo, Tianhao Li, Tianyi Tang, Wenbiao
  Yin, Xingzhang Ren, Xinyu Wang, Xinyu Zhang, Xuancheng Ren, Yang Fan, Yang
  Su, Yichang Zhang, Yinger Zhang, Yu~Wan, Yuqiong Liu, Zekun Wang, Zeyu Cui,
  Zhenru Zhang, Zhipeng Zhou, and Zihan Qiu.
\newblock Qwen3 technical report, 2025.
\newblock URL \url{https://arxiv.org/abs/2505.09388}.

\bibitem[Zhang et~al.(2025{\natexlab{a}})Zhang, Zhang, Ma, Zhang, and
  Guo]{zhang2025lazyevictionlaggedkveviction}
Haoyue Zhang, Hualei Zhang, Xiaosong Ma, Jie Zhang, and Song Guo.
\newblock Lazyeviction: Lagged kv eviction with attention pattern observation
  for efficient long reasoning, 2025{\natexlab{a}}.
\newblock URL \url{https://arxiv.org/abs/2506.15969}.

\bibitem[Zhang et~al.(2025{\natexlab{b}})Zhang, Zhu, Sun, Luo, Qiao, Du, Zheng,
  Chen, and Zhang]{zhang-etal-2025-lightthinker}
Jintian Zhang, Yuqi Zhu, Mengshu Sun, Yujie Luo, Shuofei Qiao, Lun Du,
  Da~Zheng, Huajun Chen, and Ningyu Zhang.
\newblock {L}ight{T}hinker: Thinking step-by-step compression.
\newblock In Christos Christodoulopoulos, Tanmoy Chakraborty, Carolyn Rose, and
  Violet Peng, editors, \emph{Proceedings of the 2025 Conference on Empirical
  Methods in Natural Language Processing}, pages 13307--13328, Suzhou, China,
  November 2025{\natexlab{b}}. Association for Computational Linguistics.
\newblock ISBN 979-8-89176-332-6.
\newblock \doi{10.18653/v1/2025.emnlp-main.673}.
\newblock URL \url{https://aclanthology.org/2025.emnlp-main.673/}.

\bibitem[Zhang et~al.(2023)Zhang, Sheng, Zhou, Chen, Zheng, Cai, Song, Tian,
  R\\'{e}, Barrett, Wang, and Chen]{NEURIPS2023_6ceefa7b}
Zhenyu Zhang, Ying Sheng, Tianyi Zhou, Tianlong Chen, Lianmin Zheng, Ruisi Cai,
  Zhao Song, Yuandong Tian, Christopher R\\'{e}, Clark Barrett,
  Zhangyang~"Atlas" Wang, and Beidi Chen.
\newblock H2o: Heavy-hitter oracle for efficient generative inference of large
  language models.
\newblock In A.~Oh, T.~Naumann, A.~Globerson, K.~Saenko, M.~Hardt, and
  S.~Levine, editors, \emph{Advances in Neural Information Processing Systems},
  volume~36, pages 34661--34710. Curran Associates, Inc., 2023.
\newblock URL
  \url{https://proceedings.neurips.cc/paper_files/paper/2023/file/6ceefa7b15572587b78ecfcebb2827f8-Paper-Conference.pdf}.

\bibitem[Zheng et~al.(2025)Zheng, Chen, Li, Li, Zong, Shi, Xu, Song, Wong, and
  See]{zheng2025the}
Tianshi Zheng, Yixiang Chen, Chengxi Li, Chunyang Li, Qing Zong, Haochen Shi,
  Baixuan Xu, Yangqiu Song, Ginny Wong, and Simon See.
\newblock The curse of cot: On the limitations of chain-of-thought in
  in-context learning.
\newblock \emph{Transactions on Machine Learning Research}, 2025.
\newblock ISSN 2835-8856.
\newblock URL \url{https://openreview.net/forum?id=7SIrvcYNYj}.

\bibitem[Zhu et~al.(2025)Zhu, Huang, Peng, Lu, Yu, Cheng, Qiu, Huang, and
  Lin]{zhu2025autologiautomatedgenerationlogic}
Qin Zhu, Fei Huang, Runyu Peng, Keming Lu, Bowen Yu, Qinyuan Cheng, Xipeng Qiu,
  Xuanjing Huang, and Junyang Lin.
\newblock Autologi: Automated generation of logic puzzles for evaluating
  reasoning abilities of large language models, 2025.
\newblock URL \url{https://arxiv.org/abs/2502.16906}.

\end{thebibliography}

}
\end{document}